\documentclass[conference]{IEEEtran}
\usepackage{times}
\usepackage{amsmath}
\usepackage{amsthm}
\usepackage{amssymb}

\usepackage{xspace}
\usepackage{xcolor}
\usepackage{graphicx}
\usepackage{subcaption}
\usepackage{multirow}
\usepackage{makecell}

\usepackage{float}
\usepackage{caption}
\usepackage{booktabs}
\usepackage{balance}
\usepackage{tikz}
\usetikzlibrary{chains}
\usetikzlibrary{calc}
\usepackage{titlesec}

\usepackage[numbers]{natbib}
\usepackage{multicol}
\usepackage[bookmarks=true]{hyperref}

\def\X{X}
\def\Xfree{\X_{\text{free}}}

\newtheorem{theorem}{Theorem}
\newtheorem{lemma}{Lemma}

\def\optimizer{\sigma}
\def\path{p}
\def\pathprime{\path^{\prime}}

\def\pp{\path^{\prime}}

\def\R{\mathbb{R}}
\def\i{(\textit{i})\xspace}
\def\ii{(\textit{ii})\xspace}
\def\iii{(\textit{iii})\xspace}

\def\optimizer{\sigma}
\def\radius{r_{\optimizer}}

\begin{document}

\title{Approximate Topological Optimization using Multi-Mode Estimation for Robot Motion Planning}

%\author{Author Names Omitted for Anonymous Review. Paper-ID [add your ID here]}

\author{Andreas Orthey$^{1}$
and Florian T. Pokorny$^{2}$ and Marc Toussaint$^{1,3}$%
%\thanks{$^{1} $Max Planck Institute for Intelligent Systems, Stuttgart, Germany {\tt\footnotesize \{aorthey\}@is.mpg.de}}%
%\thanks{$^{2}$ Technical University of Berlin, Berlin, Germany}%
}

\twocolumn[{%
\begin{@twocolumnfalse}
\centering
\maketitle
\def\hFrac{0.24}
\def\wFrac{0.24}
\includegraphics[width=\wFrac\textwidth,height=\hFrac\textwidth]{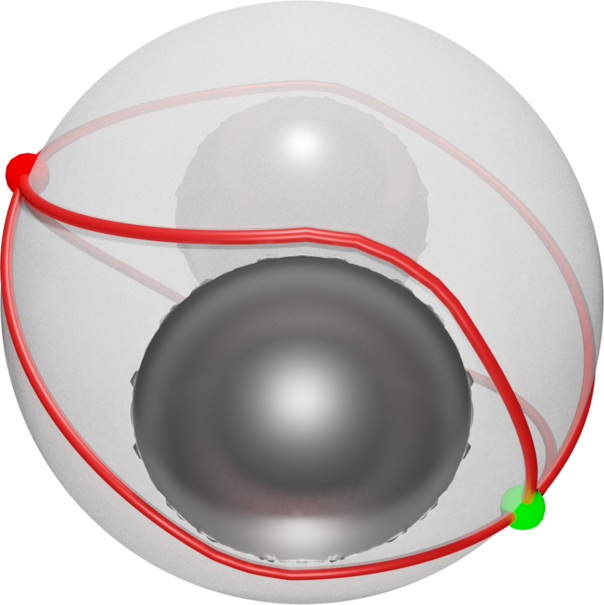}
\includegraphics[width=\wFrac\textwidth,height=0.2\textwidth]{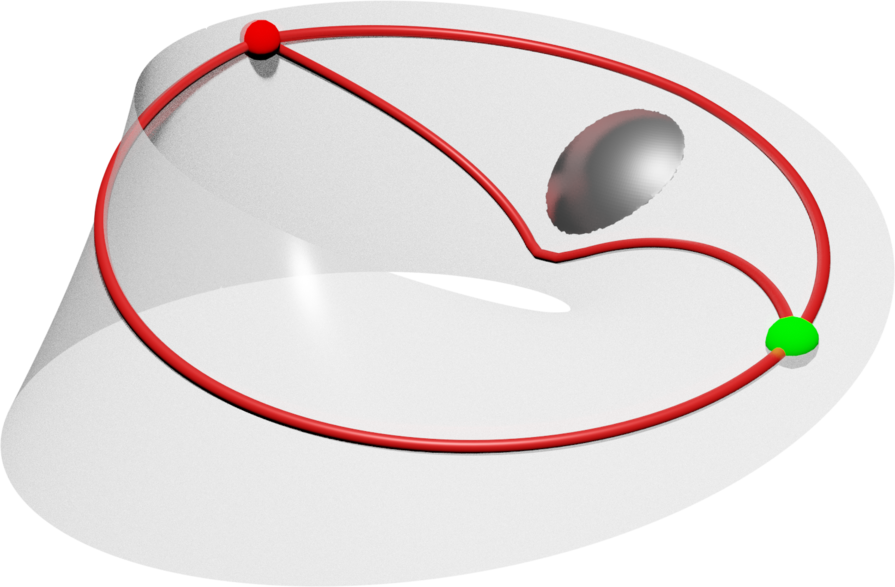}
\includegraphics[width=\wFrac\textwidth,height=\hFrac\textwidth]{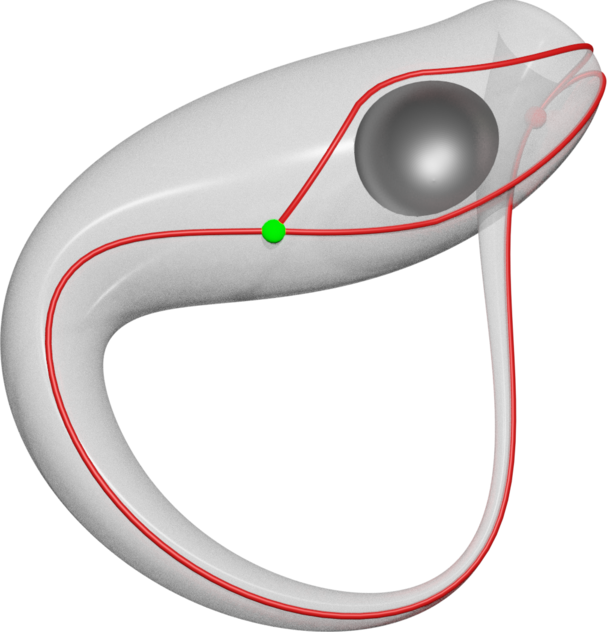}
\includegraphics[width=\wFrac\textwidth,height=\hFrac\textwidth]{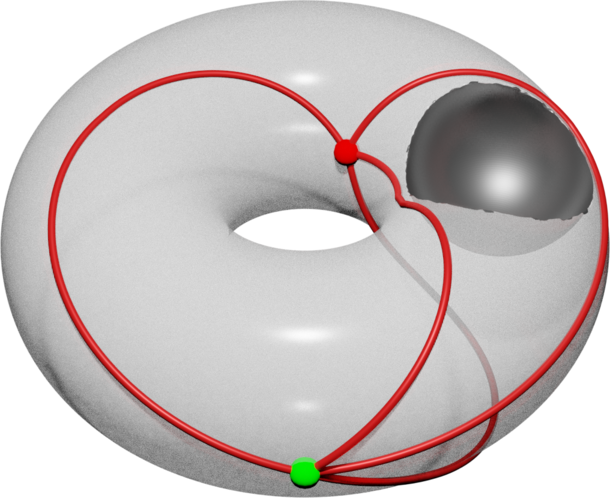}

\captionof{figure}{Initial test of multi-mode estimation on different state spaces. Problem is to move a point robot from a given start state (green) to a goal state (red) while optimizing a minimum-length cost functional. Local modes found after convergence are shown in red. \textbf{Left}: Sphere with two punctured holes (in grey) and four modes (in red). \textbf{Middle Left}: Punctured Mobius strip with three modes. \textbf{Middle Right}: Punctured Klein bottle with three modes. \textbf{Right}: Punctured torus with four modes.\label{fig:pullfigure}}
%\captionof{figure}{We generalize sparse roadmaps to fiber bundles. Here, we demonstrate this idea on the Torus $T^2 = S^1 \times S^1$ with $S^1$ being the circle. {\textbf{Left}}: Dense roadmap using probabilistic roadmap planner \cite{Karaman2011}. {\textbf{Middle}}: Sparse roadmap using sparse roadmap spanner \cite{dobson_2014}. {\textbf{Right}}: Sparse multilevel roadmap on fiber bundle $T^2 \rightarrow S^1$ using our algorithm (SMLR), which restricts sampling based on information from the lower-dimensional space $S^1$.\label{fig:pullfigure}}
\end{@twocolumnfalse}
}]

%%Hack to get thanks notes
{
  \footnotetext[1]{Max Planck Institute for Intelligent Systems, Stuttgart, Germany. Marc Toussaint thanks the MPI-IS for the Max Planck Fellowship.}%
  \footnotetext[2]{KTH Royal Institute of Technology, Stockholm, Sweden}
  \footnotetext[3]{Technical University of Berlin, Berlin, Germany. Emails:{\tt\footnotesize \{aorthey\}@is.mpg.de}, {\tt\footnotesize \{fpokorny\}@kth.se}, {\tt\footnotesize \{toussaint\}@tu-berlin.de}}
}

\begin{abstract}

In this extended abstract, we report on ongoing work towards an approximate multimodal optimization algorithm with asymptotic guarantees. Multimodal optimization is the problem of finding all local optimal solutions (modes) to a path optimization problem. This is important to compress path databases, as contingencies for replanning and as source of symbolic representations. Following ideas from Morse theory, we define modes as paths invariant under optimization of a cost functional. We develop a multi-mode estimation algorithm which approximately finds all modes of a given motion optimization problem and asymptotically converges. This is made possible by integrating sparse roadmaps with an existing single-mode optimization algorithm. Initial evaluation results show the multi-mode estimation algorithm as a promising direction to study path spaces from a topological point of view.

%\florian{later in the conclusions you focus on asymptotic convergence, but here we seem to mean convergence to within a chosen error epsilon in finite time? I suppose we might expect that for an arbitrary but fixed minima, correctly chosen error threshold the algorithm will find it in finite time under some assumptions on compactness of the space (closed and bounded) and niceness of the optimization functional. But I think for this submission, we should just claim a weaker claim according to the initial simulation work and put the theory details on "future work"}
\end{abstract}

\IEEEpeerreviewmaketitle

\section{Introduction}

We develop a new multi-mode estimation algorithm to find modes. Modes are paths invariant under optimization of a cost functional. Estimating modes is important to provide completeness guarantees to optimization algorithms, as contingencies for rapid replanning \citep{Yang2010, Pall2018, Orthey2020WAFR}, and to 
provide admissible heuristics for more complex problems \citep{Vonasek2020, Orthey2020IJRR}. This is useful in multi-robot navigation \citep{Orthey2020WAFR, Mavrogiannis2020}, to investigate long-horizon planning problems \citep{Hartmann2021TRO}, to sparsify path databases \citep{Pokorny2016}, or as symbols for high-level planning \citep{Toussaint2018}.

However, the robotics community has thus far concentrated almost exclusively on single-mode optimization algorithms. Methods include CHOMP \cite{Zucker2013}, 
TrajOpt \cite{Schulman2014}, STOMP \cite{Kalakrishnan2011}, KOMO \cite{Toussaint2009, Toussaint2014}, Kernel Projection \cite{Marinho2016},
Bayesian Optimization \cite{Vien2018} or
Gaussian Process Planning \cite{Mukadam2018}. While those optimization methods find a \emph{single mode} from an existing solution, we like to leverage them for \emph{multi-mode} estimation.

Multi-mode estimation is closely related to topological optimization.
Topological concepts like homotopy \cite{Bhattacharya2012, Bhattacharya2018},
homology \cite{Pokorny2016, Pokorny2016ICRATaskProjections} or braids
\cite{mavrogiannis_2016, Mavrogiannis2019} have been successfully applied in
robotics. Our work differs by explicitly assuming that a single-mode
optimization algorithm is given. Motivated by Morse theory \cite{Morse1934}, we
concentrate on the problem of enumerating modes. While finding modes has been
studied extensively in (evolutionary) optimization \cite{Preuss2015}, in terms
of visualization \cite{Tierny2017} and in terms of mode-relationships
\cite{Ochoa2014, Treimun2020}, there are just a few works applying this concept
to path planning \cite{Jaillet2008, Rosmann2017, Osa2020, Vonasek2020}. Our work
is complementary, in that we do multi-mode estimation but provide asymptotic guarantees. 

While we previously tackled the same problem \cite{Orthey2020RAL,
Orthey2020WAFR}, we significantly improve upon this work. In particular, our
multi-mode estimation algorithm can enumerate the modes of a given optimal
planning problem while \i asymptotically converges to all modes (see proof in
Appendix~\ref{appendix:convergence}), \ii being anytime, and \iii can use
iterative optimizers. This is achieved by having an integrated 2-stage process
combining sparse roadmaps with a single-mode path optimization method.

\tikzset{
  boxed/.style={
    rectangle, 
    rounded corners, 
    draw=black, very thick,
    text width=8em, 
    minimum height=2.5em, 
    text centered},
  arrowed/.style={
   ->},
  shift left/.style ={commutative diagrams/shift left={#1}},
  shift right/.style={commutative diagrams/shift right={#1}}
}

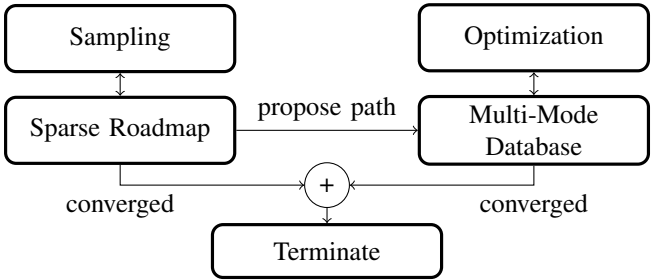
\begin{figure}
    \centering
\begin{tikzpicture}[node distance=5.5cm and 1cm]
    \node (B) [boxed] {Sparse Roadmap};
    %\node (A1) [below left=-1.8cm and 0.5cm of B] {$\X$};
    %\node (A2) [below left=-1.3cm and 0.5cm of B] {$x_I$};
    %\node (A3) [below left=-0.8cm and 0.5cm of B] {$x_G$};
    \node (C) [boxed, right of=B] {Multi-Mode Database};
    \node (B1) [boxed, above=0.3cm of B] {Sampling};
    \node (C1) [boxed, above=0.3cm of C] {Optimization};
    \draw [<->] (C) -- (C1);
    \draw [<->] (B) -- (B1);
    \node (D1) at ($(B)!0.5!(C)$){};    
    \node (D) [draw, fill=white, circle,
below=0.3cm of D1] {+};
    \node (F) [boxed,minimum height=2em, 
below=0.2cm of D] {Terminate};
    %\draw [->] (A) to[bend left=10] node[above]{Uniform Sampling}(B);
    \draw [->,yshift=0.2cm] (B) -- node[above]{propose path}(C);
    \draw [->] (B.south) |- node[below]{converged}(D);
    \draw [->] (C.south) |- node[below]{converged}(D);
    \draw [->] (D) -- (F);
    %\draw [-] (A1) -| (B.west);
    %\draw [-] (A2) -| (B.west);
    %\draw [-] (A3) -| (B.west);
\end{tikzpicture}    
\caption{Overview of multi-mode estimation. As input, we use a state space on which we grow a sparse roadmap \citep{Dobson2014}. New paths in the sparse roadmap are proposed to the multi-mode database for storage and optimization. We terminate if both the sparse roadmap and all database paths have converged.}
    \label{fig:systemfigure}\vspace{-3mm}
\end{figure}
\section{Approximate Topological Optimization}

We consider optimal motion planning problems of the form $(X, x_I, x_G, c, \optimizer)$ whereby $X$ is the state space, $x_I$ is an initial state, $x_G$ a goal state, $c$ is a cost functional of the form
\begin{equation}
    c(p) = \int_0^1 L(x,p(x),p'(x)) dx \label{eq:costfunctional}
\end{equation}
whereby $p: I \rightarrow \X$ is a path on the state space from $x_I$ to $x_G$
and $\optimizer$ is a path optimizer. We assume that $\optimizer$ is
deterministic, does not  increase path cost and asymptotically converges to a
fixed point path (See Appendix~\ref{appendix:convergence} for more details). Our
goal is to find the modes, i.e. the set of paths which are invariant under optimizer $\optimizer$ over the cost functional Eq.~\eqref{eq:costfunctional}.

\subsection{Overview of Multi-Mode Estimation}

Our general method is an integrated 2-stage process depicted in Fig.~\ref{fig:systemfigure}. We first grow a sparse roadmap \cite{Dobson2014} on the state space. Whenever we add a new edge to the roadmap, we check if this edge creates a new path from start to goal. This path is then added to the multi-mode database. Our algorithm terminates if no samples are added to the sparse roadmap for $M$ subsequent iterations \emph{and} if all paths in the database have asymptotically converged under optimizer $\optimizer$. This algorithm is anytime, i.e. it returns all current paths on premature termination.

\subsection{Sparse Roadmaps}

In the first step of our algorithm we grow a sparse roadmap on the state space. This method follows closely previous algorithms \cite{Simeon2000, Dobson2014} and uses the parameter $\Delta$ as the visibility radius to reject samples. Whenever we add a new edge to the sparse graph, we check if this edge adds a path from start to goal. This is done by first checking if the source vertex of the edge is in the same component as the start and the goal vertex of the graph. We then construct a path from start to the source vertex and another path from source vertex to the goal vertex. Both paths are then concatenated and send to the multi-mode estimation database. We say that the sparse roadmap is converged if $M$ subsequent infeasible samples have been drawn \cite{Simeon2000}.

\subsection{Multi-mode database}

The multi-mode database gets as input a stream of paths and adds those paths either to the database or uses them to update an existing path. In the first step, we apply the optimizer $\optimizer$ for one iteration. We then iterate through all database paths and check for path equivalence. Path equivalence is tested using the Hausdorff metric which is defined as
\begin{equation}
    d_H(p,\pp) = \sup_{s \in I} \inf \{ d(p(s), \pp(t)) \mid t \in I\},
\end{equation}
whereby $\sup$ is the supremum (least upper bound), $\inf$ is the infimum
(greatest lower bound) and $d$ is a metric on $\X$. We say that two paths $p,
\pp$ are equivalent if $d_H(p,\pp) < \epsilon$, whereby $\epsilon$ is a resolution parameter specific to the optimizer. If two paths are equivalent, we either update the existing path if its cost is lower or remove it. If no paths are equivalent, we add the new path to the database. After every iteration, we randomly pick a non-converged path, optimize it for one iteration, and check for path equivalence. This stage is said to be converged if all paths have converged. 

\section{Demonstrations}

As an initial test of the multi-mode estimation algorithm, we use four 2-d examples. In particular, we use punctured versions of the sphere, the Mobius strip, the Klein bottle and the torus (Fig.~\ref{fig:pullfigure}). Each of those state spaces has been implemented in the open motion planning library (OMPL) \cite{Sucan2012}, whereby we implement uniform sampling using curvature-based rejection sampling \cite{Williamson1987} and implement interpolation functions which take the gluing of the state space into account. 

Using parameters as detailed in Appendix~\ref{appendix:parameters}, we let the
multi-mode estimation run $10$ times and report on $t$, the time until
convergence in seconds, $m$, the average number of minima found, and $n$, the average
number of nodes in the sparse roadmap after convergence. For the Mobius strip,
we have $t =  8.75 \pm 1.24 $, $m= 3$ and $n= 78.3$. For the sphere, the results
are $t =  17.76 \pm 4.32 $, $m= 4$, $n= 42.1 $, for the torus $t =  16.39 \pm
1.52 $, $m= 5$, $n= 116.0$ and for the Klein bottle $t = 62.41 \pm 35.82 $,
$m=3.5$, $n= 97.2 $. The results indicate that the multi-modal estimation
algorithm robustly converges while having a low number of nodes in the sparse
roadmap. We observe that in the Klein bottle scenario, the optimizer converges
to $3$ modes in half of the cases, while converging to $4$ in the rest of them.
This is due to numerical instabilities in the optimizer, which needs to be
addressed in future work. However, our multi-mode estimation algorithm is
conservative, i.e. it reliably finds all existing modes in all cases. 

\section{Conclusion}

This extended abstract described a method to estimate modes of a robot motion
planning and optimization problem. To apply this methodology to more complex
problems, we further need to solve two problems. First, we need to scale the
algorithm up to high-dimensional systems. This can be achieved by aggregating
modes using multilevel abstractions \cite{Orthey2020WAFR, Orthey2020IJRR}. Second, while we
provided proof for asymptotic convergence in the Appendix, it is often difficult
to achieve this in practice due to jumping behavior in the optimizer. This could
potentially be alleviated by using more powerful optimizers with iteration step
guarantees. However, we believe multi-mode estimation to be a useful and
promising tool to study high-dimensional path spaces from a topological point of
view.

\bibliographystyle{plainnat}
\balance
\bibliography{bib/general}

\appendices

\balance
\section{Proof of Asymptotic Convergence\label{appendix:convergence}}

In this appendix, we prove that the
multi-mode estimation algorithm converges to all modes of a given optimization
problem, i.e. that it asymptotically converges. This result depends on the
parameters $M$ and $\Delta$ of the sparse roadmap. In particular, we prove that
multi-mode estimation converges with probability $1$ as $M$ approaches infinity
to all modes with a basin of attraction (neighborhood of paths converging to the
mode) of radius at least $\Delta$. 

%The bad news are that we can give guarantees that our algorithm converges with probability $1-p$ to all the modes which have a basin of attraction of at least $d$ (i.e. a set of paths with distance bounded by $d$ which will converge to the mode). This is a serious limitation, since $d$ is not known beforehand and could be arbitrary. 

%The good news are that both $p$ and $d$ are tune-able parameters of our algorithm, which can be made arbitrarily small, thereby giving stronger guarantees at the expanse of increased runtime. However, this gives additional freedom for applying the algorithm in a realistic problem setting.

\subsection{Notations and Assumptions}

Let $\X$ be the state space, a compact manifold (closed, bounded, locally
euclidean), $x_I \in \X$ be the start and $x_G \in \X$ be the goal state. Let us
further assume that we are given a distance function $d: \X \times \X
\rightarrow \R$ and a constraint function $\phi: \X \rightarrow \{0,1\}$, which
evaluates to zero if a state is constraint-free and to one otherwise (this
function encapsulates e.g. joint limits, self-collisions,
robot-environment-collisions or robot-robot-collisions). The constraint function
implicitly defines the free state space $\Xfree = \{\phi(x)=0\mid x \in \X\}$.
We then define the \emph{path space} $P$ to be $P=\Xfree^{I}\Bigr|_{x_I}^{x_G}$, i.e. the set of paths $p: I \rightarrow \Xfree$ which start at $x_I$ and end at $x_G$. 

Given a path space, we like to estimate modes. Recall that a mode is a path
invariant under \emph{optimization} of a \emph{cost functional}. We assume a
given cost functional $c(p)$ as introduced in Eq.~\eqref{eq:costfunctional}
which assigns a cost value to each path in our path space and we assume a path optimizer $\optimizer: P \rightarrow P$, which is a mapping taking as input a path and returning as output another path. We make additional assumptions on this optimizer. 
\begin{itemize}
    \item \textbf{Non-increasing cost}. If $c(p)$ is the cost of $p$, then $c(\optimizer(p)) \leq c(p)$.
    \item \textbf{Deterministic}. Any two applications of $\optimizer(p)$ will yield the same unique deterministic path.
    \item \textbf{Converges Asymptotically to a Fixed Point}. The path optimizer will eventually converge to a path, i.e. there exists an $N>0$ such that for any $n > N$ we have $\optimizer^n(p) = \optimizer^{n-1}(p)$ whereby $\optimizer^n$ is the $n$ times repeated application of the mapping $\optimizer$.
\end{itemize}

Note that we do not make additional assumptions such as continuity, i.e. that input and output are continuously deformable into each other (homotopic), or idempotence, i.e. that the optimizer converges after one iteration. While this slightly increases the complexity of our algorithm, it vastly increases the scope of single-mode optimizers we can use, such as shortcutting algorithms \cite{Sekhavat1998} or iterative gradient descent methods \cite{Toussaint2014}.

\subsection{Basin of Attraction}

Let $p$ be a mode of Eq.~\eqref{eq:costfunctional}. We define its \emph{basin of attraction} as the set of paths $\pathprime \in P$ such that there exists an $N>0$ such that for all $n > N$ we have $d_H(p,\optimizer^n(\pathprime)) \leq \epsilon$, i.e. $p$ is a fixed point for the optimizer $\optimizer$ when applied to path $\pathprime$. In this case, the mode is also called an attractor. 

For the theoretical treatment, it is important to quantify the size of the basin of attraction. We introduce $\radius(p)$, which we define as the \emph{radius of the basin of attraction} of mode $p$, which is the largest real number such that all paths $\pathprime$ of distance $d_H(p,\pathprime) \leq \radius(p)$ are inside the basin of attraction for optimizer $\optimizer$.

In practice, we might have to deal with optimizers which exhibit jumping-behavior, where it looks like the optimizer converged (i.e. we stay below $\epsilon$ for one iteration), while in reality the optimizer might have only made small steps and has in fact not yet converged. This can be dealt with by defining a parameter $N_E$, which is the number of steps during which the optimizer needs to stay below $\epsilon$ for us to consider the path to be converged to a fixed point. 

\subsection{Asymptotic Convergence}

We are ready to state our two theorems. The first theorem establishes that for
every mode $p$, there exists (in the limit) a path on the sparse roadmap which
has a Hausdorff distance to the mode which is upper bounded by the visibility
radius $\Delta$. In the second theorem, we then establish that all modes with a
radius of basin of attraction larger than the constant will eventually be found
with probability $1-\frac{1}{M}$. Before we state those theorem, we restate a
Lemma by Dobson and Bekris \cite{Dobson2014}, which will come in handy in our
proofs.

\begin{lemma}[Dobson and Bekris \cite{Dobson2014}]
Let $p$ be a path. Then there exists a connected series of vertices $V = \{v_1,\ldots,v_M\}$ on the sparse graph, with probability approaching $1$ as $M$ goes to infinity, such that every point $p(s)$ on the image of $p$ lies in the visibility region of at least one vertex in $V$.
\end{lemma}

This is basically a slight restatement of Lemma~1 in Dobson and Bekris \cite{Dobson2014}. We use this Lemma to prove that paths are upper bounded in the Hausdorff distance.

%\begin{proof}
%Let $p(s)$ be an arbitrary point on $p$. By definition, $p(s)$ is in $\Xfree$. Then there exists, in the limit, a vertex $v$ on the sparse graph such that $p(s)$ is inside its visibility region. Moving along the path, we can then construct a sequence of vertices $V=\{v_1,\ldots,v_M\}$ covering $p$. Finally, we need to prove that those vertices are connected by an edge. Let us move along the image of the path $p$ from $s=0$ to $s=1$. There must come a point at which we switch from a vertex $v_m$ to $v_{m+1}$. However, since the visibility regions are open sets and the image of $p$ is a closed set, there must exists a point on $p$ which simultaneously lies in both the visibility region of $v_m$ and $v_{m+1}$. By the triangle inequality, the distance of $v_m$ and $v_{m+1}$ must be upper bounded  
%\end{proof}

\begin{theorem}
\label{thm:upperbound}
Let $p$ be a mode. Then there exists a path $\pathprime$ on the sparse roadmap, with probability approaching $1$ as $M$ goes to infinity, such that $d_H(p, \pathprime)$ is upper bounded by $\Delta$ for $d_H$ being the Hausdorff distance.
\end{theorem}

\begin{proof}
By Lemma~1, $p$ is covered by a path $\pathprime$ consisting of a connected sequence of vertices. Let $p(s)$ be an element on the image of the mode. Then $p(s)$ lies in its visibility radius of some vertex $v$ on $\pathprime$ and its distance to the path is therefore upper bounded by $d(p(s), v) \leq \Delta$. Since this is true for every point, the Hausdorff distance can therefore be upper bounded by $\Delta$.
\end{proof}

Theorem~\ref{thm:upperbound} establishes that there exists, for every mode, in
the limit, a path on the sparse roadmap which is inside a $\Delta$-neighborhood
of the mode. We use this result to establish that every mode with a sufficiently bounded basin of attraction will be found in the limit.

\begin{theorem}
Let $p$ be a mode with radius of basin of attraction of $\radius(p)$. If $\radius(p) \geq \Delta$, then multi-mode estimation with optimizer $\optimizer$ will eventually converge to $p$ with probability approaching $1$ as $M$ goes to infinity.
\end{theorem}

\begin{proof}
By Theorem~\ref{thm:upperbound}, there exists, in the limit as $M$ goes to infinity, a path $\pp$ on the sparse roadmap with Hausdorff distance upper bounded by $\Delta$. Since this path is by assumption below the radius of the basin of attraction $\radius(p)$, this path will, eventually, converge to $p$ by repeated application of the mapping $\optimizer$. 
\end{proof}

Those two theorems show that we will find all modes for optimizer $\optimizer$ with probability approaching $1$ as $M$ goes to infinity and which have a basin of attraction of at least $\Delta$. Modes which belong to tight narrow passages or where the cost functional is pathological will often have small basin of attraction and are therefore not guaranteed to be found. However, we can tune both $M$ to be arbitrarily large and $\Delta$ to be arbitrarily small, thereby accommodating even ill-behaved modes and ill-behaved cost functionals. 

\section{Parameters\label{appendix:parameters}}

In practice, we need to choose additional parameters for the multi-mode
estimation, such that (1) we converge to all (desired) modes and (2) the runtime does not blow up. By trial and error, we converged to the set of parameters shown in Table~\ref{tab:parameters}, both for the sparse roadmap and for the optimizer. 

\begin{table}[H]
    \centering
    \begin{tabular}{|c|c|c|c|}
    \cline{2-4}
        \multicolumn{1}{c|}{} &Parameter & Description & Value\\
        \hline
       \multirow{2}{*}{\makecell{Sparse\\ Roadmap}} & $\Delta$ & Visibility radius of a node  &  $0.1 \mu$ \\
       & $M$ & Number of subsequent failures  & $5000$ \\
       & $t$ & Stretch factor & $3$ \\
       \hline
       \multirow{3}{*}{Optimizer} & $\epsilon$ & Accuracy for path equivalence & $0.3$ \\
       & $\epsilon_{E}$ & Threshold for convergence error & $1\mathrm{e}{-2}$ \\
       & $N_{E}$ & \makecell[t]{Number of sub-threshold iterations\\ to declare convergence} & $10$ \\
       \hline
    \end{tabular}
    \caption{Parameters used in the multi-mode estimation algorithm. The constant $\mu$ is the measure of the state space.}
    \label{tab:parameters}
\end{table}

\end{document}